\pdfoutput=1

\documentclass{new_tlp}
\usepackage{mathptmx}
\usepackage{amsmath}
\usepackage{amssymb}
\usepackage{xspace}
\usepackage[ruled,linesnumbered,vlined]{algorithm2e}
\usepackage{booktabs}
\usepackage{url}

\usepackage{enumitem}

\usepackage{graphicx}

\let\proof\relax

\usepackage{amsthm}
\newtheorem{theorem}{Theorem}[section]

\newtheorem{proposition}{Proposition}[section]

\theoremstyle{definition}
\newtheorem{example}{Example}

\usepackage{listings}
\makeatletter
\lst@Key{countblanklines}{true}[t]{\lstKV@SetIf{#1}\lst@ifcountblanklines}
\lst@AddToHook{OnEmptyLine}{%
    \lst@ifnumberblanklines\else%
    \lst@ifcountblanklines\else%
    \advance\c@lstnumber-\@ne\relax%
    \fi%
    \fi}
\makeatother
\lstdefinelanguage{asp}{
    breakatwhitespace=true,
    captionpos=b,
    numbers=left,
    numbersep=5pt,
    numberblanklines=false,
    countblanklines=false,
    commentstyle=\color{gray},
    frame=bt, framexbottommargin=5pt, framextopmargin=5pt,
    aboveskip=5pt, belowskip=5pt,
    abovecaptionskip=10pt
}
\lstnewenvironment{asp}{
	\lstset{
		language=asp,
		showstringspaces=false,
        keepspaces=true,
		formfeed=\newpage,
		tabsize=4,
		numbers=none,
		breaklines=true,
		literate={~} {$\sim$}{1},
    frame=none,
	}
}{}

\newcommand{\tr}{\ensuremath{\mathit{tr}}\xspace}
\newcommand{\sk}{\ensuremath{\mathit{sk}}\xspace}
\newcommand{\he}{\ensuremath{\mathit{he}}\xspace}
\newcommand{\UP}{\ensuremath{\mathit{UP}}\xspace}
\newcommand{\z}{\ensuremath{\phantom{0}}}

\title[Modal Logic S5 Satisfiability in Answer Set Programming]{Modal Logic S5 Satisfiability\\ in Answer Set Programming}

\author[M. Alviano, S. Batsakis, and G. Baryannis]{
    MARIO ALVIANO\\
    University of Calabria, Italy\\
    \email{alviano@mat.unical.it}
    \and
    SOTIRIS BATSAKIS\\
    Technical University of Crete, Greece and University of Huddersfield, UK\\
    \email{s.batsakis@hud.ac.uk}
    \and
    GEORGE BARYANNIS\\
    School of Computing and Engineering, University of Huddersfield, UK\\
    \email{g.bargiannis@hud.ac.uk}}

\date{April 2021}

\begin{document}

\maketitle
  \begin{abstract}
Modal logic S5 has attracted significant attention and has led to several practical applications, owing to its simplified approach to dealing with  nesting modal operators. Efficient implementations for evaluating satisfiability of S5 formulas commonly rely on Skolemisation to convert them into propositional logic formulas, essentially by introducing copies of propositional atoms for each set of interpretations (possible worlds).
This approach is simple, but often results into large formulas that are too difficult to process, and therefore more parsimonious constructions are required.
In this work, we propose to use Answer Set Programming for implementing such constructions, and in particular for identifying the propositional atoms that are relevant in every world by means of a reachability relation.
The proposed encodings are designed to take advantage of other properties such as entailment relations of subformulas rooted by modal operators.
An empirical assessment of the proposed encodings shows that the reachability relation is very effective and leads to comparable performance to a state-of-the-art S5 solver based on SAT, while entailment relations are possibly too expensive to reason about and may result in overhead. This paper is under consideration for acceptance in TPLP.
  \end{abstract}

  
\section{Introduction}\label{sec:introduction}


Modal logics have been extensively studied in the context of knowledge representation and reasoning and, more broadly, artificial intelligence, providing the ability to qualify truth according to different modalities, most commonly the alethic modalities of necessity and possibility. Several practical applications of modal logics have been proposed, such as analysing syntax structures and natural language semantics in linguistics~\cite{Moss2007}, belief and trust in multi-agent systems~\cite{Liau2003} and weak models of distributed computing~\cite{Hella2015}. Applications are even more extensive if we consider other modalities such as epistemic logic, used to reason about knowledge and belief states of agents~\cite{moses2008reasoning}, deontic logic, used in legal representation and reasoning~\cite{Batsakis2018} and temporal logics, which form the basis of significant formal verification and model checking approaches, beginning with Amir Pnueli's seminal work~\cite{Pnueli1977}.

The main approach for defining the semantics of modal logics is Kripke (or relational) semantics~\cite{kripke1959completeness}, with formulas interpreted according to a set of possible worlds, in which formulas can be true or false, and an accessibility relation that determines which worlds are accessible assuming that we are at one of them. Kripke semantics have been used to model several different modal logic systems, including S5, in which the accessibility relation is actually an equivalence relation, since it satisfies the reflexive, symmetric and transitive properties. An additional attractive feature of S5 is that sequences of modal operators applied on a formula are simplified by pruning all but the one closest to the formula. 

Significant research effort has been devoted to solving the satisfiability problem for S5 (S5-SAT), which has been proven to be NP-complete  \cite{ladner1977computational}, exploring various established proof methods such as tableau~\cite{Gore1999}, resolution~\cite{Nalon2017} and propositional satisfiability (SAT). As argued by \citeN{huang2019solving}, SAT-based methods for modal logic satisfiability have increasingly shown potential, both due to improvements in modern SAT solvers and because of their ability to learn from conflicts and balance between guessing and reasoning. Reducing S5-SAT to SAT commonly involves Skolemisation, introducing fresh Boolean variables to denote truth values in different possible worlds, which often results in quite large formulas that are difficult to process. Recently proposed S5-SAT algorithms tackle this issue in different ways: notably S52SAT~\cite{caridroit2017sat} relies on improved upper bounds on the number of possible worlds required to find a model as well as structural caching, while \textsc{S5Cheetah}~\cite{huang2019solving} benefits from formula normalisation and a graph-based representation of conflicts between modalised literals.

In this work, we propose to use Answer Set Programming (ASP)~\cite{lifschitz2019answer,brewka2011answer} to provide parsimonious propositional representations of S5 formulas and rely on ASP solvers for the satisfiability problem of S5. The motivation behind this proposal lies in the close relationship between ASP and SAT and the readability and configurability afforded by ASP encodings due to their logic programming nature~\cite{Baryannis2018,Baryannis2020}. We first provide a complete ASP encoding of the S5 normal form introduced by~\citeN{huang2019solving}. Then, we introduce an optimisation to the base ASP encoding exploiting the fact that only some propositional atoms are reachable through propagation based on the existence of the possibility modal operator. We also look at additional modifications based on known entailment and conflict relations between modal operators. 

The main contributions of this paper can be summarised as follows:
\begin{itemize}
    \item We propose an approach to convert an S5 formula to an ASP program that is applicable to any formula that has been normalised according to the S5 normal form introduced by~\citeN{huang2019solving}.
    \item An optimised version of the base ASP encoding is also proposed, that benefits from materialising only those atoms that are reachable. We also explore additional modifications: modal subsets (i.e. detecting if modal operators are applied on literals which are subsets of literals with the same modal operators) and modal conflicts (i.e. detecting if specific combinations of modal operators lead to conflicts).
    \item \textsc{S5py}, an ASP-based solver for modal logic S5 satisfiability is implemented and evaluated. The solver automatically converts arbitrary S5 formulas into their equivalent S5 normal forms and then to ASP programs. The solver may apply some of the optimisations/modifications depending on runtime options in order to find a model.
\end{itemize}

The proposed ASP-based S5-SAT solver is experimentally evaluated to determine the relative efficiency of different encodings and the benefits of each optimisation, confirming that the reachable atoms grounding optimisation achieves the most significant improvement, while modifications based on entailment and conflicts may result in overhead due to them being potentially expensive to reason about. Additionally, we compare different versions of the encoding (including some or all of the modifications) with the state-of-the-art solver \textsc{S5Cheetah} looking at the trade-off between the number of solved instances and the required execution time showing that the proposed solver achieves comparable performance.

The rest of this paper is organised as follows. Section~\ref{sec:background} summarises background knowledge on modal logics, S5 and S5 normal form, as well as related research on S5 satisfiability solving. Section~\ref{sec:relations} provides a formalisation of the reachability relation for the proposed optimisation and other concepts relevant to the modal subsets and modal conflicts modifications, while Section~\ref{sec:implementation} details all ASP encodings, including the base encoding of an S5 formula in S5 normal form and the proposed optimisations and modifications. Section~\ref{sec:evaluation} presents and discusses the implementation and experimental evaluation of the ASP-based S5-SAT solver, Finally, Section~\ref{sec:conclusions} concludes and points out future research directions.

\section{Background and Related Work}\label{sec:background}

In this section we briefly summarise preliminaries about modal logic S5, beginning with its syntax and semantics in Section~\ref{sec:modalS5} and the S5 normal form introduced by~\citeN{huang2019solving} in Section~\ref{sec:normal}. This is followed by a presentation of existing reasoning approaches for S5 in Section~\ref{sec:S5solvers}.


\subsection{Modal Logic S5}\label{sec:modalS5}

S5 extends propositional logic with the modal operators $\Box$ for encoding \emph{necessity} and $\Diamond$ for encoding \emph{possibility}. The language is defined by the grammar
\begin{equation}\label{eq:s5-grammar}
    \phi := p \mid \neg\phi \mid \phi \wedge \phi \mid \phi \vee \phi \mid \Box \phi \mid \Diamond \phi
\end{equation}
where $p$ is a \emph{propositional atom} among those of a fixed countably infinite set $\mathcal{A}$.
Moreover, logical connectives for implication and equivalence are used as syntactic sugar with the usual meaning, i.e. 
$\phi \rightarrow \psi := \neg\phi \vee \psi$ and
$\phi \leftrightarrow \psi := (\phi \rightarrow \psi) \wedge (\psi \rightarrow \phi)$, for every pair of formulas $\phi$ and $\psi$.
The \emph{complement} of a propositional literal is defined as usual, i.e. $\overline{p} = \neg p$ and $\overline{\neg p} = p$ for all $p \in \mathcal{A}$, and the notation is naturally extended to sets of propositional literals.

The semantics of S5 formulas is given by \emph{Kripke structures}, that is, non-empty sets of \emph{worlds} connected by an accessibility relation;
the accessibility relation can be assumed to be total for S5 formulas \cite{DBLP:journals/apal/Fitting99}, so for the purposes of this paper only the set of worlds will be used.
A \emph{world} is an interpretation of propositional logic, that is, a function $I$ assigning a truth value of either 0 (false) or 1 (true) to every propositional atom in $\mathcal{A}$.
Interpretations are usually represented by the sets of propositional atoms that are assigned a value of true.
Let $\mathbf{I}$ be the list $[I_0,\ldots,I_n]$ of worlds, for some $n \geq 0$, and let $0 \leq i \leq n$.
The \emph{satisfiability relation} $\models$ for S5 formulas is defined as follows:
\begin{itemize}
    \item $(\mathbf{I},i) \models p$ iff $I_i(p) = 1$
    \item $(\mathbf{I},i) \models \neg\phi$ iff $(\mathbf{I},i) \not\models \phi$
    \item $(\mathbf{I},i) \models \phi \wedge \psi$ iff $(\mathbf{I},i) \models \phi$ and $(\mathbf{I},i) \models \psi$
    \item $(\mathbf{I},i) \models \phi \vee \psi$ iff $(\mathbf{I},i) \models \phi$ or $(\mathbf{I},i) \models \psi$
    \item $(\mathbf{I},i) \models \Box\phi$ iff $(\mathbf{I},j) \models \phi$ for all $j \in [0..n]$
    \item $(\mathbf{I},i) \models \Diamond\phi$ iff $(\mathbf{I},j) \models \phi$ for some $j \in [0..n]$
\end{itemize}
The \emph{satisfiability problem} associated with S5 is the following:
given an S5 formula $\phi$, is there a list $\mathbf{I} = [I_0,\ldots,I_n]$ (for some $n \geq 0$) such that $(\mathbf{I}, 0) \models \phi$?

\subsection{S5 Normal Form}\label{sec:normal}

A \emph{propositional literal} $\ell$ is either a propositional atom or its negation.
A \emph{$\Box$-literal} has the form $\Box(\ell_1 \vee \cdots \vee \ell_n)$, where $n \geq 1$ and $\ell_1,\ldots,\ell_n$ are propositional literals.
A \emph{$\Diamond$-literal} has the form $\Diamond(\ell_1 \wedge \cdots \wedge \ell_n)$, where $n \geq 1$ and $\ell_1,\ldots,\ell_n$ are propositional literals.
An \emph{S5-literal} is a propositional literal, a $\Box$-literal, or a $\Diamond$-literal.
A disjunction of S5-literals is called an \emph{S5-clause}.
A formula $\phi$ is in \emph{S5 normal form (S5-NF)} if $\phi$ is a conjunction of S5-clauses.
Let $\mathit{atoms}(\phi)$ and $\mathit{lits}(\phi)$ denote the sets of propositional atoms and literals occurring in $\phi$, respectively.

Let $\phi,\psi$ be S5 formulas, and $p$ be a propositional atom.
Let $\psi[\phi/p]$ denote the formula obtained by \emph{substituting} every occurrence of $p$ in $\psi$ with $\phi$.
Every S5 formula $\psi$ can be transformed into an equi-satisfiable S5-NF formula by applying the transformation $\mathit{tr}$ below to its subformulas \cite{huang2019solving}.
Let $p$ be a \emph{fresh} propositional atom (i.e. an atom not occurring in $\psi$), $n \geq m \geq 1$, and for all $i \in [1..n]$ let $\phi_i$ be an S5 formula and $\odot_i \in \{\Box, \Diamond\}$.
Transformation $\mathit{tr}$ is defined by the following transformation rules:
\begin{enumerate}
    \item 
    $\neg\neg\phi_1 \overset{\tr}{\mapsto} \phi_1$, 
    $\neg(\phi_1 \wedge \cdots \wedge \phi_n) \overset{\tr}{\mapsto} \neg\phi_1 \vee \cdots \vee \neg\phi_n$, 
    $\neg(\phi_1 \vee \cdots \vee \phi_n) \overset{\tr}{\mapsto} \neg\phi_1 \wedge \cdots \wedge \neg\phi_n$,\\
    $\neg\Box\phi_1 \overset{\tr}{\mapsto} \Diamond\neg\phi_1$, and
    $\neg\Diamond\phi_1 \overset{\tr}{\mapsto} \Box\neg\phi_1$;
    \item
    $\odot_1 \cdots \odot_n\phi_n \overset{\tr}{\mapsto} \odot_n\phi_n$, 
    $\Box(\phi_1 \wedge \cdots \wedge \phi_n) \overset{\tr}{\mapsto} \Box\phi_1 \wedge \cdots \wedge \Box\phi_n$, and\\
    $\Diamond(\phi_1 \vee \cdots \vee \phi_n) \overset{\tr}{\mapsto} \Diamond\phi_1 \vee \cdots \vee \Diamond\phi_n$;
    \item $\Box(\phi_1 \vee \cdots \vee \phi_m \vee \odot_{m+1}\phi_{m+1} \vee \cdots \vee \odot_n\phi_n) \overset{\tr}{\mapsto} \Box(\phi_1 \vee \cdots \vee \phi_m) \vee \odot_{m+1}\phi_{m+1} \vee \ldots \vee \odot_n\phi_n$;
    \item $\Diamond(\phi_1 \wedge \cdots \wedge \phi_m \wedge \odot_{m+1}\phi_{m+1} \wedge \cdots \wedge \odot_n\phi_n) \overset{\tr}{\mapsto} \Diamond(\phi_1 \wedge \cdots \wedge \phi_m) \wedge \odot_{m+1}\phi_{m+1} \wedge \ldots \wedge \odot_n\phi_n$;
    \item $\psi[\phi_1 \vee \cdots \vee \phi_m \vee (\phi_{m+1} \wedge \cdots \wedge \phi_n)/p'] \overset{\tr}{\mapsto} \psi[(\phi_1 \vee \cdots \vee \phi_m \vee p)/p'] \wedge \Box(\neg p \vee \phi_{m+1}) \wedge \cdots \wedge \Box(\neg p \vee \phi_n)$;
    \item $\psi[\phi_1 \wedge \cdots \wedge \phi_m \wedge (\phi_{m+1} \vee \cdots \vee \phi_n)/p'] \overset{\tr}{\mapsto} \psi[\phi_1 \wedge \cdots \wedge \phi_m \wedge p/p'] \wedge \Box(\neg p \vee \phi_{m+1} \cdots \vee \phi_n)$.
\end{enumerate}
Note that transformation rules in items~1--4 can be applied working top-down on the tree structure of the formula, and since they are local to subformulas we opted for a simpler notation (e.g. the rule for double negation elimination in item~1 is actually $\psi[\neg\neg\phi_1/p'] \overset{\tr}{\mapsto} \psi[\phi_1/p']$, that is, all occurrences of $\neg\neg\phi_1$ in the input formula $\psi[\neg\neg\phi_1/p']$ are transformed into $\phi_1$).
Moreover, the above transformation rules are intended to be applied in the provided order, and in particular rules~5--6 are defined based on this assumption and are applied working bottom-up on the tree structure of the formula unless the formula is already in S5-NF.
In fact, these transformation rules are introducing a fresh propositional atom $p$ to compactly represent a conjunction or a disjunction, like in Tseitin's transformation, but using less clauses as in the algorithm of~\citeN{DBLP:journals/jsc/PlaistedG86}. This is achieved by exploiting the fact that the formula is already in negation normal form thanks to rules in item~1 (there is no negation in the path connecting the subformula to the root).
Also note that after applying rules~5--6, rule~3 becomes applicable again (and possibly in turn the first rule in item~2).
In the following, if not otherwise said, formulas are assumed to be in S5-NF.

\subsection{Modal Logic Solvers}\label{sec:S5solvers}
As explained in the introduction, modal logics have many practical applications which have led researchers to explore efficient tools for reasoning problems such as satisfiability. Early approaches that are adopted in practice for reasoning over several modal logics are translation to first-order logic \cite{ohlbach1991semantics} and resolution-based methods \cite{auffray1990strategies} applicable to S4, K, Q, T and K4 modal logics. The \textsc{KsP} solver \cite{nalon2017ksp} is a recent contribution that adopts the resolution-based method and is applicable to the propositional multimodal logic $K_{n}$. Tableau-based reasoners applicable for modal logics include: \textsc{Spartacus} \cite{gotzmann2010spartacus} for basic modal logic K; \textsc{LoTREC} \cite{gasquet2005lotrec} which covers a wide range of modal and description logics; and the \textsc{LCK} implementation presented by \citeN{abate2007cut} which is also applicable to several modal and temporal logics. 

Performance improvements in modern SAT solvers has led to the development of SAT-based solvers for modal logics. Examples include \textsc{KSAT} \cite{giunchiglia2000sat} for modal logic K, which is shown to outperform contemporary tableau and translation-based approaches, and \textsc{Km2SAT} \cite{sebastiani2009automated}, which is applicable to $K_{m}$ modal logic and \textit{ALC} description logic. \textsc{InKreSAT} \cite{kaminski2013inkresat} is based on an incremental SAT solver and can deal with modal logics K, T, S4 and K4. \textsc{S52SAT} \cite{caridroit2017sat} goes one step beyond these solvers as it addresses the satisfiability problem for S5, which is not supported by most of the other SAT-based solvers. In addition, \textsc{S52SAT} estimates an upper bound of possible worlds using the \emph{diamond degree} of a formula and this tight bound offers a noticeable performance gain.

Following this line of research, the SAT-based \textsc{S5Cheetah} solver \cite{huang2019solving} estimates an upper bound on possible worlds by applying the \emph{graph colourability heuristic}, which is used for identifying non interacting $\Diamond$-literals that can be materialised in the same world. \textsc{S5Cheetah} can be considered a state-of-the-art S5 solver as it outperforms both the \textsc{LCK} implementation by \citeN{abate2007cut} and the \textsc{S52SAT} solver~\cite{caridroit2017sat}. For this reason, our proposed ASP-based solving approach, formalised in the following section and presented in Section~\ref{sec:implementation}, is compared against \textsc{S5Cheetah} in the experiments discussed in Section~\ref{sec:evaluation}.

\section{S5 Satisfiability Checking via Skolemisation}\label{sec:relations}

In this section, we formalise the proposed approach to checking satisfiability of S5-NF formulas on which the ASP encoding in Section~\ref{sec:implementation} is based. S5 formulas, in general, can be transformed into monadic first-order formulas associating every propositional atom $p$ with a (unary) predicate $p$;
arguments of those predicates are worlds, and therefore variables are universally quantified for $\Box$-literals and existentially quantified for $\Diamond$-literals.
Moreover, since we are interested in the satisfiability problem, the existential closure of the formula is actually checked for satisfiability.
Finally, existential variables can be eliminated by Skolemisation, that is, by replacing them with fresh constants.

Formally, let $\phi$ be an S5-NF formula, and let us fix an enumeration $\Box\psi^\Box_1, \ldots, \Box\psi^\Box_m$, $\Diamond\psi^\Diamond_1, \ldots,$ $\Diamond\psi^\Diamond_n$ of its $\Box$- and $\Diamond$-literals, for some $m \geq 0$ and $n \geq 0$.
Let $\psi(x)$ denote the monadic first-order formula obtained from $\psi$ by adding argument $x$ to all propositional atoms occurring in $\psi$.
The \emph{Skolemisation} of $\phi$, denoted $\sk(\phi)$, is defined by the following transformation rules (applied in the given order):
\begin{enumerate}
    \item $\Box\psi^\Box_i \overset{\sk}{\mapsto} \forall x\ \psi^\Box_i(x)$, for all $i \in [1..m]$;
    \item $\Diamond\psi^\Diamond_i \overset{\sk}{\mapsto} \psi^\Diamond_i(i)$, for all $i \in [1..n]$;
    \item $p \overset{\sk}{\mapsto} p(0)$, for the remaining propositional literals.
\end{enumerate}
Note that the second transformation rule uses the Skolem constant $i$ to Skolemise formula \linebreak $\exists x\ \psi^\Diamond_i(x)$, and the third transformation rule uses the Skolem constant $0$ to Skolemise free variables (which are subject to the existential closure). Hence, in order to check satisfiability of an S5-NF formula $\phi$, one can equivalently check satisfiability of the \emph{Herbrand expansion} of $\sk(\phi)$, in order to take advantage of modern SAT solvers, for instance.

There are several observations on the structure of S5-NF formulas that can be taken into account in order to improve on the set of clauses that result from the Herbrand expansion, and the rest of this section focuses on these. First of all, fresh propositional atoms must be introduced to represent $\Box$- and $\Diamond$-literals, as well as clauses to impose equivalence between every such fresh propositional atom and the associated S5-literal;
as already observed in Section~\ref{sec:normal}, implications are actually sufficient \cite{DBLP:journals/jsc/PlaistedG86}.
If $b_1,\ldots,b_m$ and $d_1,\ldots,d_n$ are such fresh propositional atoms, the following implications are encoded in clauses:
\begin{itemize}
    \item $b_i \rightarrow \psi^\Box_i(0) \wedge \cdots \wedge \psi^\Box_i(n) \equiv (b_i \rightarrow \psi^\Box_i(0)) \wedge \cdots \wedge (b_i \rightarrow \psi^\Box_i(n))$, for all $i \in [1..m]$;
    \item $d_i \rightarrow \psi^\Diamond_i(i)$, for all $i \in [1..n]$.
\end{itemize}
Let $\he(\phi)$ be such a set of clauses, that is, the Herbrand expansion of $\sk(\phi)$.

\begin{example}\label{ex:he}
Let $\phi$ be $p \wedge \Box(p \vee q) \wedge (\Diamond(p \wedge q) \vee \Diamond(\neg p \wedge \neg q))$.
Hence, $\he(\phi)$ is the set of clauses encoding the following formulas:
$p(0) \wedge b_1 \wedge (d_1 \vee d_2)$,
$b_1 \rightarrow p(0) \vee q(0)$,
$b_1 \rightarrow p(1) \vee q(1)$,
$b_1 \rightarrow p(2) \vee q(2)$,
$d_1 \rightarrow p(1) \wedge q(1)$, and
$d_2 \rightarrow \neg p(2) \wedge \neg q(2)$.
There are two lists of distinct worlds that satisfy $\phi$, namely $\mathbf{I_1} = [\{p\}, \{p,q\}]$ and $\mathbf{I_2} = [\{p,q\}]$;
they are represented by the following models of $\he(\phi)$:
$\{b_1, d_1, p(0), p(1), q(1)\} \cup X$ and
$\{b_1, d_1, p(0), q(0), p(1), q(1)\} \cup X$, where $X \in 2^{\{p(2), q(2)\}}$.
\hfill$\blacksquare$
\end{example}

\begin{proposition}
For every S5-NF formula $\phi$, $\he(\phi)$ is equi-satisfiable to $\phi$.
\end{proposition}

A second observation concerns the fact that worlds are relevant only if the associated $\Diamond$-literal is true, and therefore for all $i \in [1..m]$ and $j \in [1..n]$, the clause encoding $b_i \rightarrow \psi^\Box_i(j)$ can be replaced by a clause encoding $b_i \wedge d_j \rightarrow \psi^\Box_i(j)$.
Actually, if $\psi^\Diamond_j(0)$ is true, there is again no need to consider world $j$.
Hence, clauses encoding $b_i \rightarrow \psi^\Box_i(j)$ in $\he(\phi)$ are replaced by clauses encoding $b_i \wedge d_j \wedge \neg \mathit{implied}_j \rightarrow \psi^\Box_i(j)$, $\mathit{implied}_j \leftrightarrow \psi^\Diamond_j(0)$ and $\mathit{implied}_j \rightarrow d_j$ to obtain $\mathit{full}(\phi)$, where $\mathit{implied}_j$ is a fresh propositional atom.

\begin{example}[Continuing Example~\ref{ex:he}]
Clauses in $\mathit{full}(\phi)$ encode the following formulas:
$p(0) \wedge b_1 \wedge (d_1 \vee d_2)$,
$b_1 \rightarrow p(0) \vee q(0)$,
$b_1 \wedge d_1 \wedge \neg \mathit{implied}_1 \rightarrow p(1) \vee q(1)$,
$\mathit{implied}_1 \leftrightarrow p(0) \wedge q(0)$,
$\mathit{implied}_1 \rightarrow d_1$,
$b_1 \wedge d_2 \wedge \neg \mathit{implied}_2 \rightarrow p(2) \vee q(2)$,
$\mathit{implied}_2 \leftrightarrow p(0) \wedge q(0)$,
$\mathit{implied}_2 \rightarrow d_2$,
$d_1 \rightarrow p(1) \wedge q(1)$, and
$d_2 \rightarrow \neg p(2) \wedge \neg q(2)$.
The models of $\mathit{full}(\phi)$ are
$\{b_1, d_1, p(0), p(1), q(1)\} \cup X_2$ and
$\{b_1, d_1, \mathit{implied}_1, p(0), q(0)\} \cup X_1 \cup X_2$, where $X_i \in 2^{\{p(i), q(i)\}}$ for $i \in [1..2]$.
\hfill$\blacksquare$
\end{example}

\begin{theorem}
For every S5-NF formula $\phi$, $\mathit{full}(\phi)$ is equi-satisfiable to $\he(\phi)$.
\end{theorem}
\begin{proof}
We show equi-satisfiablity of $\Gamma = \{b_i \rightarrow \psi^\Box_i(0),$ $b_i \rightarrow \psi^\Box_i(j)\}$ and $\Gamma' = \{b_i \rightarrow \psi^\Box_i(0),$ $b_i \wedge d_j \wedge \neg \mathit{implied}_j \rightarrow \psi^\Box_i(j),$ $\mathit{implied}_j \leftrightarrow \psi^\Diamond_j(0),$ $\mathit{implied}_j \rightarrow d_j\}$, from which the claim follows.
If $I \models \Gamma$, then $I' = (I \cap \mathit{atoms}(\Gamma)) \cup \{\mathit{implied}_j, d_j \mid I \models \psi^\Diamond_j(0)\}$ is such that $I' \models \Gamma'$.
As for the other direction, let $I \models \Gamma'$ and $I \models b_i \wedge (\neg d_j \vee \mathit{implied}_j)$ (otherwise $I' \models \Gamma'$ by construction).
From $b_i \rightarrow \psi^\Box_i(0)$ we have that $I \models \psi^\Box_i(0)$, and we can copy world $0$ into world $j$ to construct a model $I'$ for $\Gamma$:
$I' = (I \cap \mathit{atoms}(\Gamma)) \setminus \{p(j) \in I\} \cup \{p(j) \mid p(0) \in I\}$.
\end{proof}

The next observation is more involved and regards the possibility for a world associated with a $\Diamond$-literal to reuse the assignment provided by world $0$ to satisfy some $\Box$-literals.
Any interpretation $I$ satisfying $\mathit{full}(\phi)$ is such that, for all $i \in [1..m]$ and $j \in [1..n]$, $I \models b_i \rightarrow \psi^\Box_i(0)$ and $I \models b_i \wedge d_j \wedge \neg \mathit{implied}_j \rightarrow \psi^\Box_i(j)$.
Hence, for all $i \in [1..m]$ such that $I(b_i) = 1$, $I \models \psi^\Box_i(j)$ for $j \in [0..n]$ if $j$ is world $0$ or a relevant world (ie. $j \geq 1$ and $I \models d_j \wedge \neg \mathit{implied}_j$).
It turns out that world $0$ witnesses the possibility to jointly satisfy all true $\Box$-literals, and therefore the other worlds can focus on $\Box$-literals that may be reached by performing unit propagation from the associated $\Diamond$-literals (resulting into a reachability relation).
Other $\Box$-literals can be satisfied by reusing the assignment provided by world $0$.
Formally, for a set $L$ of literals
\begin{align}
    \label{eq:up}
    \UP(L) :={} & L \cup \bigcup\left\{\mathit{lits}(\psi^\Box_i) \setminus \{\overline{\ell}\} \mid \ell \in L,\ i \in [1..m], \overline{\ell} \in \mathit{lits}(\psi^\Box_i)\right\}\\
    \label{eq:bj}
    B_j :={} & \left\{i \in [1..m] \mid \overline{\UP \Uparrow \mathit{lits}(\psi^\Diamond_j)} \cap \mathit{lits}(\psi^\Box_i) \neq \emptyset\right\}
\end{align}
Intuitively, $\UP(L)$ is the set of literals that may be used to satisfy every $\psi^\Box_i$ affected by the assignment of $L$, $\UP \Uparrow \mathit{lits}(\psi^\Diamond_j)$ is the set of literals reached in this way from the literals in $\psi^\Diamond_j$, and $B_j$ represents the set of $\Box$-literals involved in this computation.
Let $\mathit{reach}(\phi)$ be obtained from $\mathit{full}(\phi)$ by removing clauses encoding $b_i \wedge d_j \wedge \neg \mathit{implied}_j \rightarrow \psi^\Box_i(j)$, for all $i \in [1..m]$ and $j \in [1..n]$ such that $i \notin B_j$.

\begin{example}
Let $\phi$ be $\Box(p \vee q) \wedge \Diamond p \wedge \Diamond\neg p$.
Clauses in $\mathit{full}(\phi)$ encode the following formulas:
\begin{equation*}
\begin{array}{rrrr}
b_1 \wedge d_1 \wedge d_2 \qquad
b_1 \rightarrow p(0) \vee q(0) &
d_1 \rightarrow p(1) &
d_2 \rightarrow \neg p(2) \\
b_1 \wedge d_1 \wedge \neg \mathit{implied}_1 \rightarrow p(1) \vee q(1) &
\mathit{implied}_1 \leftrightarrow p(0) &
\mathit{implied}_1 \rightarrow d_1 \\
b_1 \wedge d_2 \wedge \neg \mathit{implied}_2 \rightarrow p(2) \vee q(2) &
\mathit{implied}_2 \leftrightarrow \neg p(0) &
\mathit{implied}_2 \rightarrow d_2 \\
\end{array}
\end{equation*}
In order to construct $\mathit{reach}(\phi)$, we have to compute sets $B_1$ and $B_2$;
indeed, $\phi$ contains two $\Diamond$-literals, namely $\Diamond p$ and $\Diamond\neg p$.
Let us first determine the sets of reached literals from $\{p\}$ and $\{\neg p\}$ according to \eqref{eq:up}:
$\UP(\{p\}) = \{p\}$ --- note that $\overline{p} \notin \mathit{lits}(p \vee q)$;
$\UP(\{\neg p\}) = \{\neg p, q\}$ --- note that $\overline{\neg p} \in \mathit{lits}(p \vee q)$ and therefore literals in $\mathit{lits}(p \vee q) \setminus \{\overline{\neg p}\} = \{q\}$ are added to $\UP(\{\neg p\})$;
$\UP(\{\neg p, q\}) = \{\neg p, q\}$ --- note that $\overline{q} \notin \mathit{lits}(p \vee q)$ and therefore no other literal is added to $\UP(\{\neg p, q\})$.
Hence, we have that $\UP \Uparrow \mathit{lits}(\{p\}) = \{p\}$ and $\UP \Uparrow \mathit{lits}(\{\neg p\}) = \{\neg p, q\}$.
Now using \eqref{eq:bj}, $B_1 = \emptyset$ and $B_2 = \{1\}$, that is, $\Diamond p$ does not interact with the $\Box$-literal, while $\Diamond \neg p$ interacts with the $\Box$-literal.
Accordingly, $\mathit{reach}(\phi)$ is obtained from $\mathit{full}(\phi)$ by removing clauses encoding $b_1 \wedge d_1 \wedge \neg \mathit{implied}_1 \rightarrow p(1) \vee q(1)$.
In fact, such a formula can be satisfied by assigning to $q(1)$ the same truth value of $q(0)$.
For example, if $I$ is a model of $\mathit{reach}(\phi)$ such that $I(q(0)) = 1$, then $I \cup \{q(1)\}$ is a model of $\mathit{full}(\phi)$;
similarly, if $I$ is a model of $\mathit{reach}(\phi)$ such that $I(q(0)) = 0$, then $I \setminus \{q(1)\}$ is a model of $\mathit{full}(\phi)$.
\hfill$\blacksquare$
\end{example}

\begin{theorem}
For every S5-NF formula $\phi$, $\mathit{reach}(\phi)$ is equi-satisfiable to $\mathit{full}(\phi)$.
\end{theorem}
\begin{proof}
$I \models \mathit{full}(\phi)$ implies $I \models \mathit{reach}(\phi)$ because $\mathit{reach}(\phi) \subseteq \mathit{full}(\phi)$.
As for the other direction, let $I \models \mathit{reach}(\phi)$ be such that $I \models b_i \wedge d_j \wedge \neg\mathit{implied}_j \wedge \neg\psi^\Box_i(j)$ for some $i \in [1..m]$ and $j \in [1..n]$ such that $i \notin B_j$ (otherwise $I \models \mathit{full}(\phi)$).
Since $b_i \rightarrow \psi^\Box_i(0)$ belongs to $\mathit{reach}(\phi)$, we have that $I \models \psi^\Box_i(0)$, and we can copy a portion of world $0$ into world $j$ to construct a model $I'$ for $\mathit{reach}(\phi)$ such that $I' \models \psi^\Box_i(j)$:
\begin{itemize}[leftmargin=3em,labelsep=1em]
\item
$L = \mathit{atoms}(\psi^\Box_i(j)) \setminus \mathit{atoms}(\psi^\Diamond_j(j))$;
\item
$I' = I \setminus \{p(j) \in L\} \cup \{p(j) \in L \mid p(0) \in I\}$.
\end{itemize}
By reiterating the process, we end up with a model of $\mathit{full}(\phi)$.
\end{proof}

Other observations pertain to entailment relations between $\Box$- and $\Diamond$-literals:
\begin{enumerate}
    \item $\psi^\Box_i$ and $\psi^\Diamond_j$ cannot be jointly satisfied if $\overline{\mathit{lits}(\psi^\Box_i)} \subseteq \mathit{lits}(\psi^\Diamond_j)$;
    \item $\psi^\Box_i$ entails $\psi^\Box_j$ if $\mathit{lits}(\psi^\Box_i) \subseteq \mathit{lits}(\psi^\Box_j)$;
    \item $\psi^\Diamond_i$ entails $\psi^\Diamond_j$ if $\mathit{lits}(\psi^\Diamond_i) \supseteq \mathit{lits}(\psi^\Diamond_j)$.
\end{enumerate}

\begin{example}
The following are small examples of the above entailment relations:
\begin{enumerate}[leftmargin=3em,labelsep=1em]
\item
Formula $\Box(\neg p \vee \neg q) \wedge \Diamond(p \wedge q \wedge s)$ is unsatisfiable;
note that $\overline{\{\neg p, \neg q\}} \subseteq \{p, q, s\}$.

\item
Whenever $(\mathbf{I},0) \models \Box(p \vee q)$, also $(\mathbf{I},0) \models \Box(p \vee q \vee s)$ holds;
note that $\{p, q\} \subseteq \{p, q, s\}$.

\item
Similarly, $(\mathbf{I},0) \models \Diamond(p \wedge q \wedge s)$ implies $(\mathbf{I},0) \models \Diamond(p \wedge q)$;
note that $\{p, q, s\} \supseteq \{p, q\}$.
\end{enumerate}
The provided encodings can be enriched to represent such entailment relations.
\hfill$\blacksquare$
\end{example}

For $\Gamma \in \{\mathit{full}(\phi), \mathit{reach}(\phi)\}$, let $\mathit{conflicts}(\Gamma,\phi)$ be the set of clauses obtained from $\Gamma$ by adding clauses of the form $\neg b_i \vee \neg d_j$ for all $i \in [1..m]$ and $j \in [1..n]$ such that $\overline{\mathit{lits}(\psi^\Box_i)} \subseteq \mathit{lits}(\psi^\Diamond_j)$;
let $\mathit{boxes}(\Gamma,\phi)$ be the propositional formula obtained from $\Gamma$ by adding clauses of the form $\neg b_i \vee b_j$ for all $i \in [1..m]$ and $j \in [1..m]$ such that $\mathit{lits}(\psi^\Box_i) \subseteq \mathit{lits}(\psi^\Box_j)$;
let $\mathit{diamonds}(\Gamma,\phi)$ be the propositional formula obtained from $\Gamma$ by replacing the clauses encoding $\mathit{implied}_j \leftrightarrow \psi^\Diamond_j(0)$ by clauses encoding $\mathit{implied}_j \leftrightarrow \psi^\Diamond_j(0) \vee d_{i_1} \vee \cdots \vee d_{i_k}$ for all $j \in [1..n]$, where $\{i_1, \ldots, i_k\} = \{i \mid \mathit{lits}(\Gamma^\Diamond_i) \supseteq \mathit{lits}(\psi^\Diamond_j)\}$.

\begin{theorem}
For every S5-NF formula $\phi$, and $\Gamma \in \{\mathit{full}(\phi), \mathit{reach}(\phi)\}$, the following sets of clauses are equi-satisfiable:
$\Gamma$, $\mathit{conflicts}(\Gamma,\phi)$, $\mathit{boxes}(\Gamma,\phi)$, and $\mathit{diamonds}(\Gamma,\phi)$.
\end{theorem}
\begin{proof}
Immediate from the entailment relations.
\end{proof}

\section{Modal Logic Encoding in Answer Set Programming}\label{sec:implementation}

This section presents an ASP implementation of the propositional theories introduced in the previous section.
Common to all theories is the relational encoding of the S5-NF formula $\phi$ to be processed.
Specifically, $\phi$ is encoded by the following facts:
\begin{itemize}
    \item \lstinline|atom(p)|, for every propositional atom $p$ occurring in $\phi$;
    \item \lstinline|box(b)|, \lstinline|pos_box(b,p$_i$)|, and \lstinline|neg_box(b,p$_j$)|, for every $\Box$-literal of $\phi$ of the form \linebreak $\Box(p_1 \vee \cdots \vee p_m \vee \neg p_{m+1} \vee \cdots \vee \neg p_n)$, with $n \geq 1$ and $n \geq m \geq 0$, and all $i \in [1..m]$ and $j \in [m+1..n]$, where \lstinline|b| is an identifier for the $\Box$-literal;
    \item \lstinline|diamond(d)|, \lstinline|pos_diamond(d,p$_i$)|, and \lstinline|neg_diamond(d,p$_j$)|, for every $\Diamond$-literal of $\phi$ of the form $\Diamond(p_1 \wedge \cdots \wedge p_m \wedge \neg p_{m+1} \wedge \cdots \wedge \neg p_n)$, with $n \geq 1$ and $n \geq m \geq 0$, and all $i \in [1..m]$ and $j \in [m+1..n]$, where \lstinline|d| is an identifier for the $\Diamond$-literal;
    \item \lstinline|clause(c)|, \lstinline|pos_clause(c,lit$_i$)|, and \lstinline|neg_clause(c,p$_j$)|, for every S5-clause of $\phi$ of the form $\ell_1 \vee \cdots \vee \ell_m \vee \neg p_{m+1} \vee \cdots \vee \neg p_n$, with $n \geq 1$ and $n \geq m \geq 0$, and all $i \in [1..m]$ and $j \in [m+1..n]$, where \lstinline|c| is an identifier for the S5-clause and each \lstinline|lit$_i$| is the identifier of the associated S5-literal $\ell_i$.
\end{itemize}
Let $\Pi_\mathit{re}(\phi)$ denote the \emph{relational encoding} of $\phi$, that is, the above facts.

\begin{example}
Let $\phi$ be $p \wedge \Box(p \vee q) \wedge (\Diamond(p \wedge q) \vee \Diamond(\neg p \wedge \neg q))$. $\Pi_\mathit{re}(\phi)$ contains the following facts:
\begin{asp}
   atom(p).          clause(c1).         clause(c2).         clause(c3).
   atom(q).      pos_clause(c1,p).   pos_clause(c2,b1).  pos_clause(c3,d1).
    box(b1).        diamond(d1).        diamond(d2).     pos_clause(c3,d2).
pos_box(b1,p).  pos_diamond(d1,p).  neg_diamond(d2,p).
pos_box(b1,q).  pos_diamond(d1,q).  neg_diamond(d2,q).
\end{asp}
\hfill$\blacksquare$
\end{example}

The basic encoding materialises a full copy of the propositional atoms in all worlds and introduces a world for every $\Diamond$-literal in the input S5-NF formula $\phi$.
Let $\Pi_\mathit{full}$ be the ASP program comprising the following rules:
\begin{asp}
$r_{0\z}:\quad$ world(D,D) :- diamond(D).
$r_{1\z}:\quad$ {true(X)} :- box(X).
$r_{2\z}:\quad$ {true(X)} :- diamond(X).
$r_{3\z}:\quad$ {true(X)} :- atom(X).
$r_{4\z}:\quad$ {true(X,W)} :- world(W,_), atom(X).
$r_{5\z}:\quad$ :- clause(C); not true(X) : pos_clause(C,X); true(X) : neg_clause(C,X).
$r_{6\z}:\quad$ :- box(B), true(B); not true(X) : pos_box(B,X); true(X) : neg_box(B,X).
$r_{7\z}:\quad$ :- world(W,D); box(B), true(B), diamond(D), true(D), not implied(D);
$\phantom{r_{7\z}:}\quad$    not true(X,W) : pos_box(B,X); true(X,W) : neg_box(B,X).
$r_{8\z}:\quad$ implied(D) :- diamond(D);     true(X) : pos_diamond(D,X);
$\phantom{r_{7\z}:}\quad$                           not true(X) : neg_diamond(D,X).
$r_{9\z}:\quad$ :- diamond(D), implied(D), not true(D).
$r_{10}:\quad$ :- pos_diamond(D,X); true(D), not implied(D); world(W,D), not true(X,W).
$r_{11}:\quad$ :- neg_diamond(D,X); true(D), not implied(D); world(W,D),     true(X,W).
$r_{12}:\quad$ need(W) :- world(W,D), true(D), not implied(D).
$r_{13}:\quad$ :- atom(X), world(W,_), not need(W), true(X,W).
\end{asp}
Rule $r_0$ asserts that every $\Diamond$-literal is associated with its own world, and is possibly replaced in other encodings to let some $\Diamond$-literals share the same world.
Rules $r_2$ and $r_3$ define the search space for $\Box$-literals and $\Diamond$-literals, that is, each of them can be assumed either true or false.
Rule $r_3$ defines the search space for world 0, that is, every propositional atom can be either true or false; this assumes that the set of worlds is non-empty, in line with the assumption made for Kripke structures.
similarly, rule $r_4$ defines the search space for other worlds.
Rule $r_5$ imposes that all S5-clauses of $\phi$ are satisfied, and rule $r_6$ requires that every $\Box$-literal $\Box\psi$ assumed to be true is such that $\psi$ is true in world 0;
similarly, rule $r_7$ requires that $\psi$ is true in all worlds associated with $\Diamond$-literals assumed to be true (and not implied).
In fact, rule $r_8$ defines a $\Diamond$-literal as implied if it is true in world 0, and rule $r_9$ additionally enforces its truth to reduce the search space.
Rules $r_{10}$ and $r_{11}$ enforce truth of every $\Diamond$-literal (assumed to be true and not implied) in the associated world, and rule $r_{12}$ defines such worlds as needed;
indeed, rule $r_{13}$ enforces falsity of all propositional atoms in worlds that are not needed, again to reduce the search space.

$\Pi_\mathit{full} \cup \Pi_\mathit{re}(\phi)$ is an ASP implementation of $\mathit{full}(\phi)$ and has some stable model if and only if there is a set $\mathbf{I} = \{I_0, \ldots, I_n\}$ such that $(\mathbf{I},0) \models \phi$.
However, it materialises several propositional atoms that can be avoided.
So, a second encoding can be so designed as to limit propositional atoms in every world by those reachable from the associated diamonds.
Let $\Pi_\mathit{reach}$ be the ASP program obtained from $\Pi_\mathit{full}$ by removing rules $r_4$ and $r_7$, and by adding the following rules:
\begin{asp}
$r_{14}:\quad$ {true(Y,W)} :- world(W,D), pos_diamond(D,X), lrl(X,p,Y,_).
$r_{15}:\quad$ {true(Y,W)} :- world(W,D), neg_diamond(D,X), lrl(X,n,Y,_).
$r_{16}:\quad$ lrl(X,pos,X,pos) :- atom(X), pos_diamond(_,X).
$r_{17}:\quad$ lrl(X,neg,X,neg) :- atom(X), neg_diamond(_,X).
$r_{18}:\quad$ lrl(X,PX,Z,pos) :- lrl(X,PX,Y,neg); pos_box(B,Y); pos_box(B,Z), Z!=Y.
$r_{19}:\quad$ lrl(X,PX,Z,neg) :- lrl(X,PX,Y,neg); pos_box(B,Y); neg_box(B,Z).
$r_{20}:\quad$ lrl(X,PX,Z,pos) :- lrl(X,PX,Y,pos); neg_box(B,Y); pos_box(B,Z).
$r_{21}:\quad$ lrl(X,PX,Z,neg) :- lrl(X,PX,Y,pos); neg_box(B,Y); neg_box(B,Z), Z!=Y.
$r_{22}:\quad$ lrb(X,P,B) :- lrl(X,P,Y,neg); pos_box(B,Y). 
$r_{23}:\quad$ lrb(X,P,B) :- lrl(X,P,Y,pos); neg_box(B,Y).
$r_{24}:\quad$ reach_box(W,B) :- world(W,D), pos_diamond(D,X); lrb(X,pos,B). 
$r_{25}:\quad$ reach_box(W,B) :- world(W,D), neg_diamond(D,X); lrb(X,neg,B).
$r_{26}:\quad$ :- world(W,D), diamond(D), true(D), not implied(D); reach_box(W,B);
$\phantom{r_{26}:}\quad$    true(B); not true(X,W) : pos_box(B,X); true(X,W) : neg_box(B,X).
\end{asp}
Above, \lstinline|lrl| stands for \emph{literal reaches literal}, and \lstinline|lrb| stands for \emph{literal reaches box}.
Rules $r_{16}$--$r_{21}$ define the \emph{reach} relation introduced in Section~\ref{sec:relations} for literals occurring in some $\Diamond$-literal of $\phi$ --- essentially, set $\UP(L)$ in (\ref{eq:up}).
Rules $r_{22}$--$r_{23}$ detect for each propositional literal $\ell$ the $\Box$-literals that contain a literal reached by $\ell$, and rules $r_{24}$--$r_{25}$ computes for every world the $\Box$-literals reached by the associated $\Diamond$-literals --- essentially, sets $B_j$ in (\ref{eq:bj}).
Within such relations, the search space of every world is restricted to the reached propositional literals (rules $r_{14}$--$r_{15}$), and $\Box$-literals are enforced only if actually reached (rule $r_{26}$).
$\Pi_\mathit{reach} \cup \Pi_\mathit{re}(\phi)$ is an ASP implementation of $\mathit{reach}(\phi)$.

Let $\Pi_\mathit{conflicts}(\phi)$ extend $\Pi_\mathit{reach}$ with the rule
\begin{asp}
$r_{27}:\quad$ :- box_diamond_conflict(B,D); true(B), true(D).
\end{asp}
and the following facts:
\lstinline|box_diamond_conflict(b,d')| for every $\Box$-literal $\psi$ and $\Diamond$-literal $\psi'$ occurring in $\phi$ and such that $\overline{\mathit{lits}(\psi)} \subseteq \mathit{lits}(\psi')$, where \lstinline|b| and \lstinline|d'| are the identifiers of $\psi$ and $\psi'$.
$\Pi_\mathit{conflicts}(\phi) \cup \Pi_\mathit{re}(\phi)$ is an ASP implementation of $\mathit{conflicts}(\mathit{reach}(\phi), \phi)$.

Let $\Pi_\mathit{boxes}(\phi)$ extend $\Pi_\mathit{reach}$ with the rule
\begin{asp}
$r_{28}:\quad$ :- box_subset(B,B'), true(B), not true(B').
\end{asp}
and the following facts:
\lstinline|box_subset(b,b')| for all $\Box$-literals $\psi$, $\psi'$ occurring in $\phi$ and such that $\mathit{lits}(\psi) \subseteq \mathit{lits}(\psi')$, where \lstinline|b| and \lstinline|b'| are the identifiers of $\psi$ and $\psi'$.
$\Pi_\mathit{boxes}(\phi) \cup \Pi_\mathit{re}(\phi)$ is an ASP implementation of $\mathit{boxes}(\mathit{reach}(\phi), \phi)$.

\begin{algorithm}[t]
    \caption{ComputeWorlds($\psi$: S5-NF formula)}\label{alg:worlds}

    $D :={}$ list of $\Diamond$-literals occurring in $\phi$, sorted by decreasing size\;
    $W := \emptyset$\;
    \ForEach{$\psi$ in $D$}{
        \If{there is $w \in W$ such that $\mathit{lits}(\psi) \subseteq \mathit{lits}(\psi')$ for all $\psi' \in w$}{
            $w := w \cup \{\psi\}$\tcp*{add $\psi$ to world $w$}
        }
        \Else{
            $W := W \cup \{\{\psi\}\}$\tcp*{add a new world for $\psi$}
        }
    }
    \Return{W}\;
\end{algorithm}

Let $\Pi_\mathit{diamonds}(\phi)$ extend $\Pi_\mathit{reach} \setminus \{r_0\}$ with the rule
\begin{asp}
$r_{29}:\quad$ implied(D) :- diamond_subset(D,D'), true(D').
\end{asp}
and the following facts:
\lstinline|diamond_subset(d,d')| for all $\Diamond$-literals $\psi$, $\psi'$ occurring in $\phi$ and such that $\mathit{lits}(\psi) \subseteq \mathit{lits}(\psi')$, where \lstinline|d| and \lstinline|d'| are the identifiers of $\psi$ and $\psi'$;
\lstinline|world(d,d')| for all $w \in \mathrm{ComputeWorlds}(\psi)$ with largest $\Diamond$-literal $\psi$, and all $\psi' \in w$, where \lstinline|d| and \lstinline|d'| are the identifiers of $\psi$ and $\psi'$.
$\Pi_\mathit{diamonds}(\phi) \cup \Pi_\mathit{re}(\phi)$ is an ASP implementation of $\mathit{diamonds}(\mathit{reach}(\phi), \phi)$ with an additional merging of some worlds guided by the entailment relation between $\Diamond$-literals.

\section{Evaluation}\label{sec:evaluation}

The ASP encodings presented in Section~\ref{sec:implementation} have been implemented into a new solver, \textsc{S5py}. The solver is written in Python and uses \textsc{clingo} version 5.4.0~\cite{Gebser2016} to search for answer sets. This section reports on an empirical comparison between \textsc{S5py} and \textsc{S5Cheetah} \cite{huang2019solving} on the benchmark used to assess \textsc{S5Cheetah}.
\textsc{S5py} and pointers to benchmark files are provided at \url{http://www.mat.unical.it/~alviano/ICLP2021-s5py.zip}.

The experiments were run on an Intel Xeon 2.4 GHz with 16 GB of memory. Time and memory were limited to 300 seconds and 15 GB; similar limits are
used by \citeN{huang2019solving}, with memory limit decreased by 1 GB to avoid swapping. For each instance solved within these limits, we measured the execution time and the memory usage.

\begin{table}[b]
    \centering
    \caption{Overall number of unsolved instances due to timeouts or memory-outs, average execution time (in seconds) and memory consumption (in MB) on solved instances.}\label{tab:summary}
    \begin{tabular}{rccccc}
        \toprule
        Solver (options) & Unsolved & Timeouts & Memory outs & Avg. Time & Avg. Memory\\
        \cmidrule{1-6}
        \textsc{S5py (full)} & 283 & 189 & 94 & 14.0 & 512\\
        \textsc{S5py (reach)} & \z15 & \z15 & \z0 & 12.1 & \z99\\
        \textsc{S5py (reach+all)} & \z69 & \z69 & \z0 & 15.3 & \z86\\
        \textsc{S5py (reach+conflicts)} & \z66 & \z66 & \z0 & 14.8 & \z90\\
        \textsc{S5py (reach+boxes)} & \z16 & \z16 & \z0 & 12.6 & \z99\\
        \textsc{S5py (reach+diamonds)} & \z15 & \z15 & \z0 & 12.9 & \z96\\
        \textsc{S5Cheetah} & \z30 & \z18 & 12 & 13.1 & 345\\
        \bottomrule
    \end{tabular}
\end{table}

We tested six configurations of \textsc{S5py}:
\begin{enumerate}
    \item \textsc{full}, generation of total worlds, i.e. $\Pi_\mathit{full} \cup \Pi_\mathit{re}(\phi))$;
    \item \textsc{reach}, restriction of each world to reachable propositional atoms, i.e. $\Pi_\mathit{reach} \cup \Pi_\mathit{re}(\phi))$;
    \item \textsc{reach+conflicts}, use of conflict relation, i.e. $\Pi_\mathit{conflicts}(\phi) \cup \Pi_\mathit{re}(\phi))$;
    \item \textsc{reach+boxes}, use of subset relation for $\Box$-literals, i.e. $\Pi_\mathit{boxes}(\phi) \cup \Pi_\mathit{re}(\phi))$;
    \item \textsc{reach+diamonds}, use of subset relation for $\Diamond$-literals, i.e. $\Pi_\mathit{diamonds}(\phi) \cup \Pi_\mathit{re}(\phi))$;
    \item \textsc{reach+all}, use of the three relations above.
\end{enumerate}

Aggregated results are reported in Table~\ref{tab:summary}, where it is directly evident that the generation of total worlds is often infeasible in practice.
We also observe that \textsc{S5py (reach)} and \textsc{S5Cheetah} have a similar performance in terms of solved instances, which confirms that the restriction of each world to reachable propositional literals is a meaningful alternative to the strategy implemented by \textsc{S5Cheetah}.
Finally, we observe that other relations that \textsc{S5py} can use in its ASP encodings do not provide any performance improvement, and actually the use of the conflict relation has a sensible negative impact.

\begin{figure}
    \figrule
    \includegraphics[width=\textwidth]{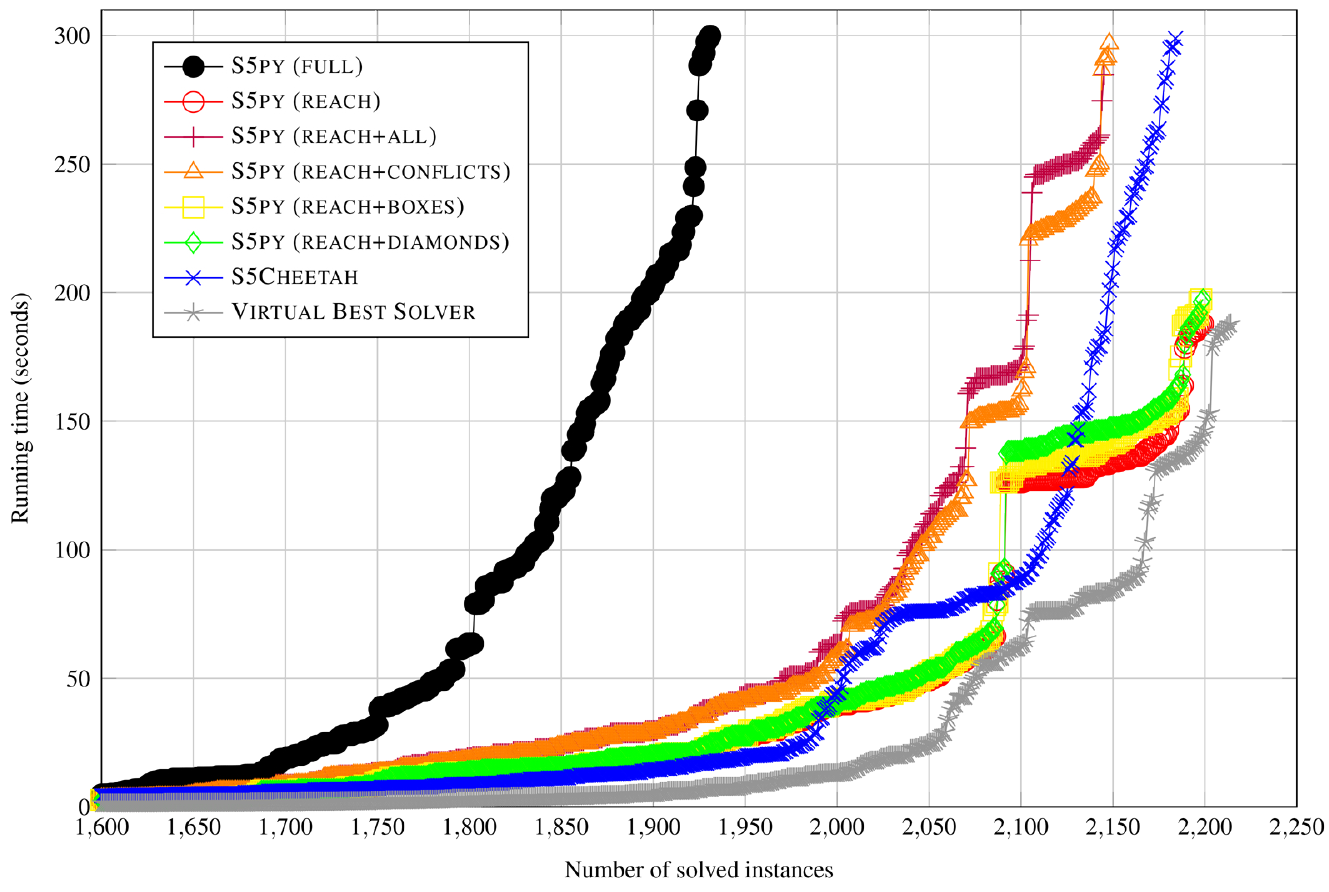}
    \centering
    \caption{Number of solved instances within a time budget}\label{fig:cactus}
    \figrule
\end{figure}

A cactus plot is shown in Figure~\ref{fig:cactus}, where for each solver the solved instances are sorted by increasing execution time. The performance of a virtual best solver is also shown, for which the execution time is the minimum execution time across all solvers. We can observe that there is no particular solver dominating over the others, though \textsc{S5py (reach)} achieves the closest performance to the virtual best solver.

\begin{figure}
    \figrule
    \includegraphics[width=\textwidth]{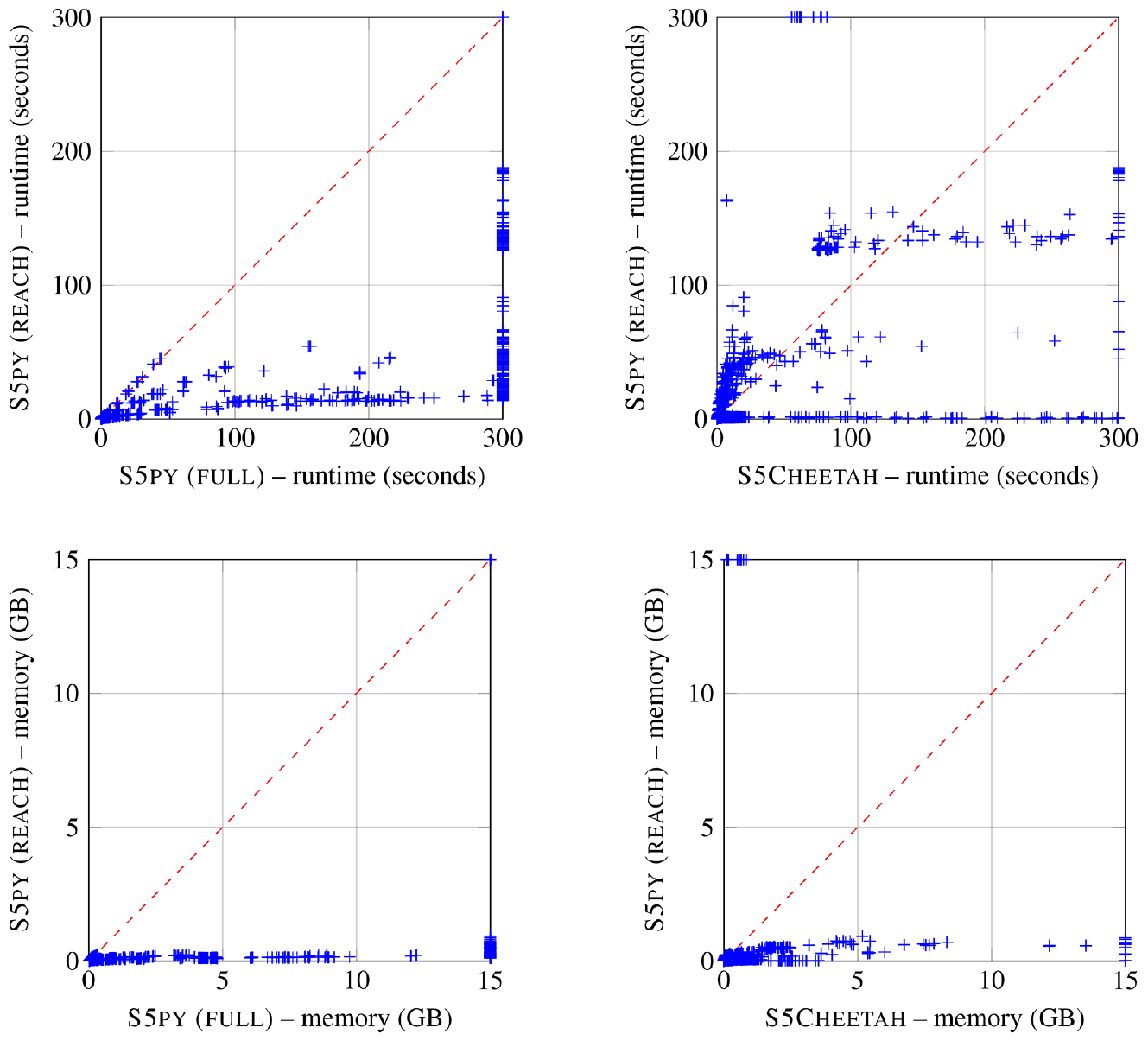}
    \centering
    \caption{Instance by instance comparison on execution time and memory: Impact of the reachability relation on \textsc{S5py} (left) and assessment with respect to the state-of-the-art \textsc{S5Cheetah} (right) in terms of execution time (top) and memory consumption (bottom). Unsolved instances normalised to the limits.}\label{fig:scatter-reachability}
    \figrule
\end{figure}

\begin{figure}
    \figrule
    \includegraphics[width=\textwidth]{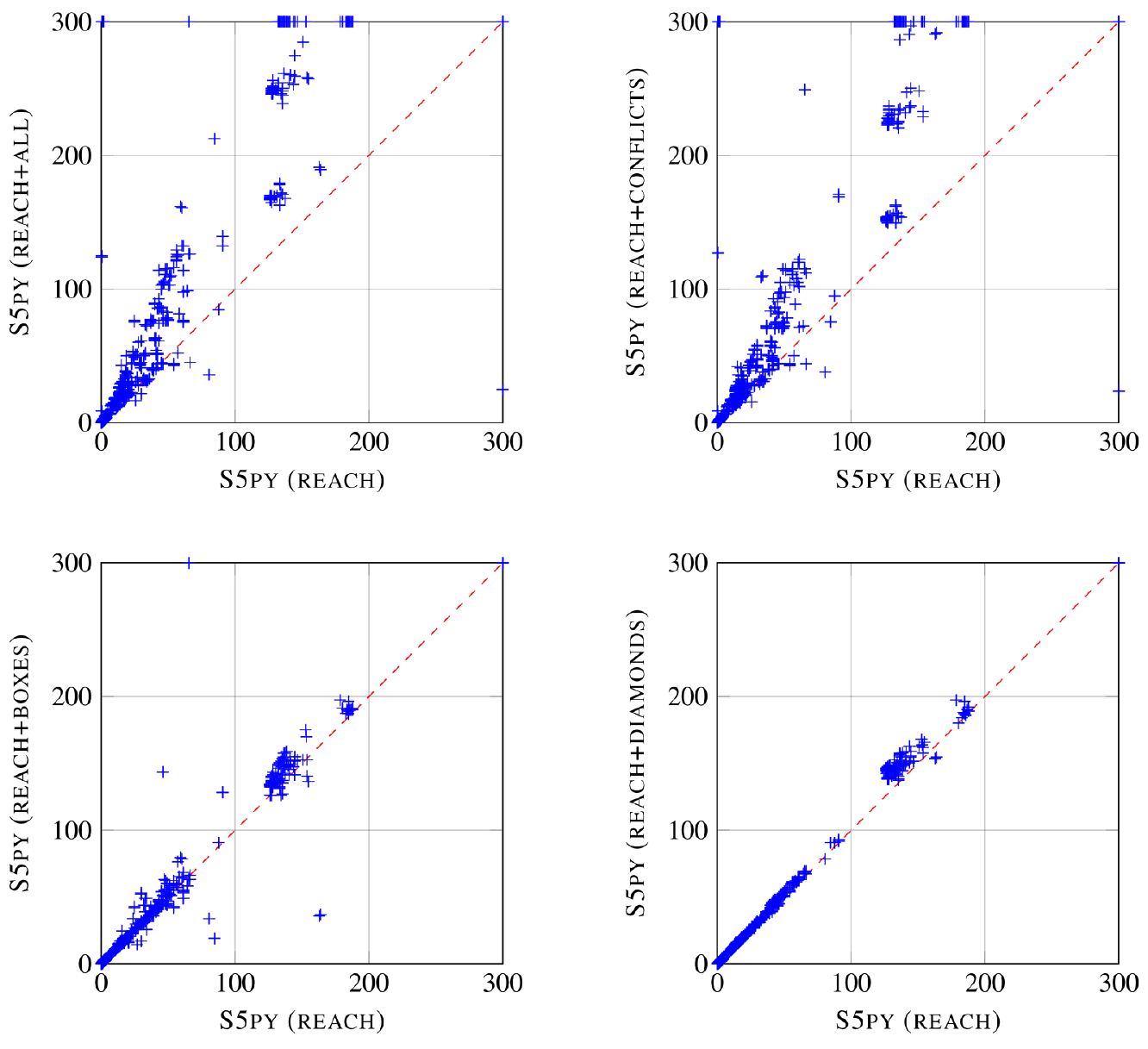}
    \centering
    \caption{Instance by instance comparison on execution time (in seconds): Impact of the entailment and conflict relations on \textsc{S5py}. Unsolved instances normalised to the limit.}\label{fig:scatter}
    \figrule
\end{figure}

An instance by instance comparison of the performance of \textsc{S5py (reach)} versus \textsc{S5py (full)} and \textsc{S5Cheetah} is provided in Figure~\ref{fig:scatter-reachability}, in terms of execution time and memory consumption.
It is quite evident that \textsc{S5py (reach)} has a uniform and sensible improvement over \textsc{S5py (full)}, as all points are essentially below the bisector. Compared to \textsc{S5Cheetah}, \textsc{S5py (reach)} requires less memory in general, as shown in the bottom two plots of Figure~\ref{fig:scatter-reachability}. However, execution time is not always in its favour, which is also evident from the cactus plot in Figure~\ref{fig:cactus}.

Scatter plots shown in Figure~\ref{fig:scatter} confirm that other relations that \textsc{S5py} can use in its ASP encodings often introduce overhead, since almost all points are above the bisector. The overhead is relatively negligible in the case of \textsc{S5py (reach+boxes)} and \textsc{S5py (reach+diamonds)} since the solver has to take into account only one additional subset relation for $\Box$-literals and $\Diamond$-literals, respectively. However, handling the conflict relation imposes a more significant overhead on the solver, since all possible pairs of $\Box$-literals and $\Diamond$-literals need to be considered. Given these results, we conclude that the usage of entailment and conflict relations is unlikely to be justified given its negative impact, while the reachability relation allows \textsc{S5py} to achieve comparable execution time to \textsc{S5Cheetah} while consuming less memory.

\section{Conclusions and Future Work}\label{sec:conclusions}

In this work, we have shown that using Answer Set Programming for implementing solvers for modal logic S5 is a both feasible and performant approach.
Experimental evaluation of the proposed encodings highlights the performance gain achievable by limiting the $\Box$-literals to satisfy in every world, to those potentially reachable (by unit propagation) from the associated $\Diamond$-literals.
In fact, the implemented solver, \textsc{S5py}, achieves a comparable performance to the state-of-the-art SAT-based solver \textsc{S5Cheetah}, with none of the two dominating the other.

Future research directions include: (a) considering combinations of the reachability-based optimisation with other optimisations proposed in literature such as the graph colouring approach implemented by \textsc{S5Cheetah}~\cite{huang2019solving}; (b) defining incremental versions of $\Pi_\mathit{reach}$ to further mitigate the negative impact of oversized propositional formulas; (c) exploring whether optimised implementations can form the basis for solvers in multi-agent settings or for related logics such as modal logic S4 and the implicational fragment of intuitionistic propositional logic (IPC)~\cite{gore_thomson_2019}.



\label{lastpage}
\end{document}